%% file: main.tex
\documentclass[10pt,twocolumn,letterpaper]{article}
\usepackage[subpreambles=true]{standalone}

\usepackage{cvpr}
\usepackage{times}
\usepackage{epsfig}
\usepackage{graphicx}
\usepackage{subcaption}
\usepackage{amsmath}
\usepackage{amssymb}
\usepackage{amsthm}
\usepackage{bm}
\usepackage{multirow}

\newtheorem{proposition}{Proposition}

\newcommand{\vv}{{\boldsymbol v}}
\newcommand{\uu}{{\boldsymbol u}}
\newcommand{\pp}{{\boldsymbol p}}
\newcommand{\ii}{{\mathbf{i}}}
\newcommand{\jj}{{\mathbf{j}}}
\newcommand{\kk}{{\mathbf{k}}}

\newcommand{\norm}[1]{\left\lVert#1\right\rVert_2}

\usepackage{authblk}
\usepackage[resetlabels]{multibib}
\newcites{supp}{References}

\usepackage[pagebackref=true,breaklinks=true,letterpaper=true,colorlinks,bookmarks=false]{hyperref}

\cvprfinalcopy 

\def\cvprPaperID{9374} 

\ifcvprfinal\fi
\begin{document}

\title{Quaternion Product Units for Deep Learning on 3D Rotation Groups}

\author[1]{Xuan Zhang}
\author[1]{Shaofei Qin}
\author[1]{Yi Xu\thanks{The corresponding author of this paper is Yi Xu (xuyi@sjtu.edu.cn). This work was supported in part by NSFC 61671298, STCSM 18DZ2270700, and 111 Plan B07022.}}
\author[2]{Hongteng Xu}
\affil[ ]{${}^{\text{1}}$MoE Key Lab of Artificial Intelligence, AI Institute \quad ${}^{\text{2}}$Infinia ML, Inc.}
\affil[ ]{${}^{\text{1}}$Shanghai Jiao Tong University, Shanghai, China \quad ${}^{\text{2}}$Duke University, Durham, NC, USA}
\affil[ ]{{\tt\small{\{floatlazer, 1105042987, xuyi\}@sjtu.edu.cn}, hongteng.xu@duke.edu}}


\maketitle
\thispagestyle{empty} 

\begin{abstract}
We propose a novel quaternion product unit (QPU) to represent data on 3D rotation groups. 
The QPU leverages quaternion algebra and the law of 3D rotation group, representing 3D rotation data as quaternions and merging them via a weighted chain of Hamilton products. 
We prove that the representations derived by the proposed QPU can be disentangled into ``rotation-invariant'' features and ``rotation-equivariant'' features, respectively, which supports the rationality and the efficiency of the QPU in theory. 
We design quaternion neural networks based on our QPUs and make our models compatible with existing deep learning models. 
Experiments on both synthetic and real-world data show that the proposed QPU is beneficial for the learning tasks requiring rotation robustness.
\end{abstract}

\input{tex/introduction.tex}

\input{tex/QPU.tex}
\input{tex/related_works.tex}
\input{tex/experiments.tex}

\section{Conclusion}
In this work, we proposed a novel quaternion product unit for deep learning on 3D rotation groups. 
This model can be used as a new module to construct quaternion-based neural networks, which presents encouraging generalization ability and flexible rotation robustness. 
Moreover, our implementation makes this model compatible with existing real-valued models, achieving end-to-end training through backpropagation. 
Besides skeleton classification, we plan to extend our QPU to more applications and deep learning.  
In particular, we have done a preliminary experiment on applying our QPU to point cloud classification task. 
We designed a quaternion representation for the neighbor point set of a centroid which first converts 3D coordinates of neighbor points into quaternion-based 3D rotations then cyclically sort them according to their rotation order around the vector from the origin to the centroid.
We designed our models based on Pointnet++~\cite{qi2017pointnet++, pytorchpointnet++} by replacing the first Set Abstraction layer with a QMLP module.
We tested our model on ModelNet40~\cite{wu20153d} and our rotation-invariant model achieved 80.1\% test accuracy.
Please refer to the supplementary file for more details of our quaternion-based point cloud representation, network architectures and experimental setups.

\textbf{Acknowledgement}
The authors would like to thank David Filliat for constructive discussions. 

{\small
\bibliographystyle{ieee_fullname}
\bibliography{reference}
}

\appendix

\include{supp/supp}

\end{document}

%% file: tex/introduction.tex
\section{Introduction}
Representing 3D data like point clouds and skeletons is essential for many real-world applications, such as autonomous driving~\cite{qi2018frustum,zhou2018voxelnet}, robotics~\cite{whelan2015elasticfusion}, and gaming~\cite{bloom2012g3d,pavllo2018quaternet}.
In practice, we often model these 3D data as the collection of points on a 3D rotation group $\mathbb{SO}(3)$, $e.g.$, a skeleton can be represented by the rotations between adjacent joints. 
Accordingly, the uncertainty hidden in these data is usually caused by the randomness on rotations. 
For example, in action recognition,
the human skeletons with different orientations may represent the same action~\cite{vemulapalli2014human,zhang2017view}; in autonomous driving, the point clouds with different directions may capture the same vehicle~\cite{zhou2018voxelnet}.  
Facing to the 3D data with such rotation discrepancies, we often require the corresponding representation methods to be robust to rotations.

\begin{figure*}[t]
    \centering
    \begin{subfigure}[b]{0.45\linewidth}
        \includegraphics[width=1\linewidth]{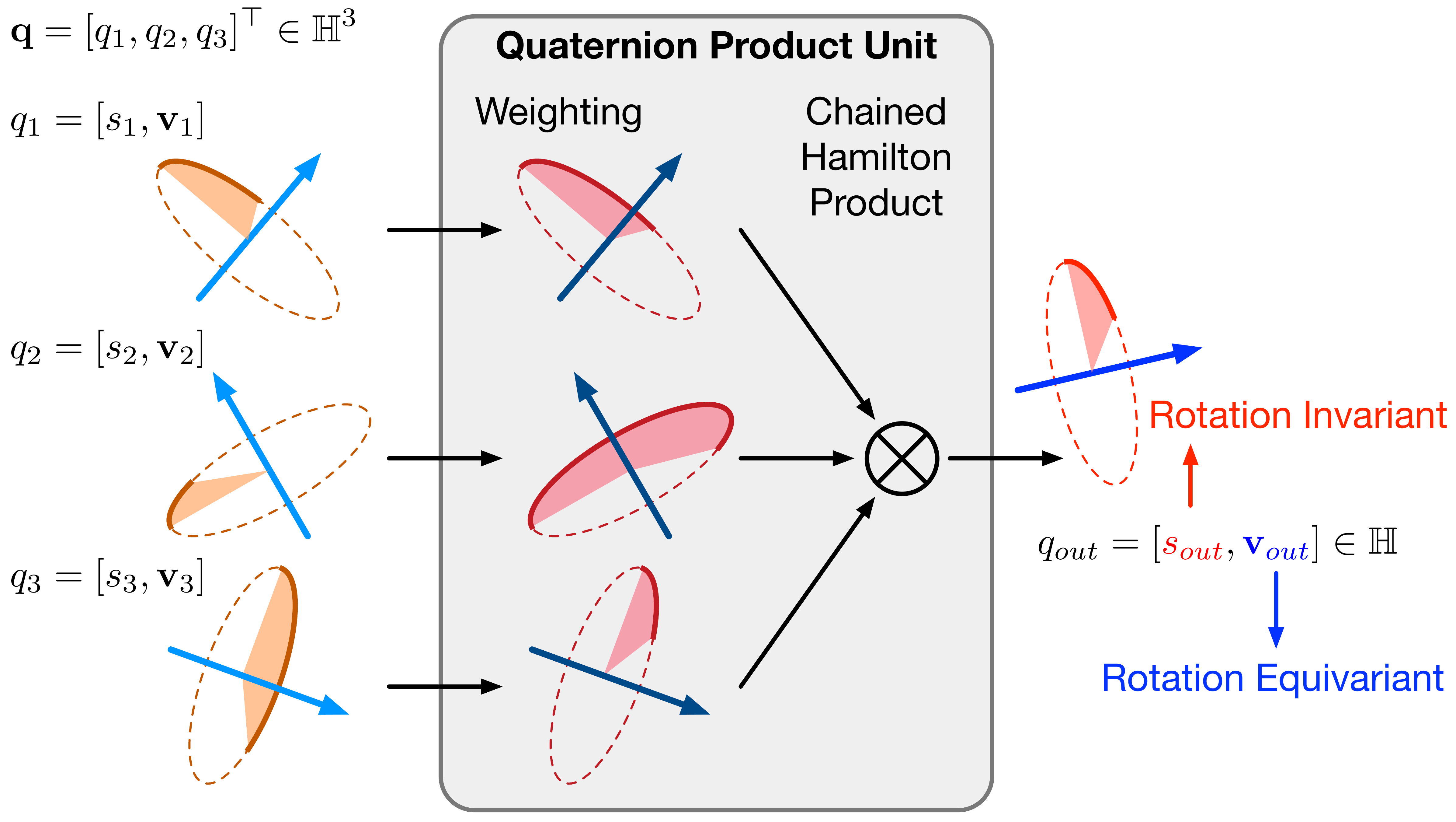}
        \caption{A single quaternion product unit}\label{fig:qpu}
    \end{subfigure}\quad
    \begin{subfigure}[b]{0.45\linewidth}
        \includegraphics[width=1\linewidth]{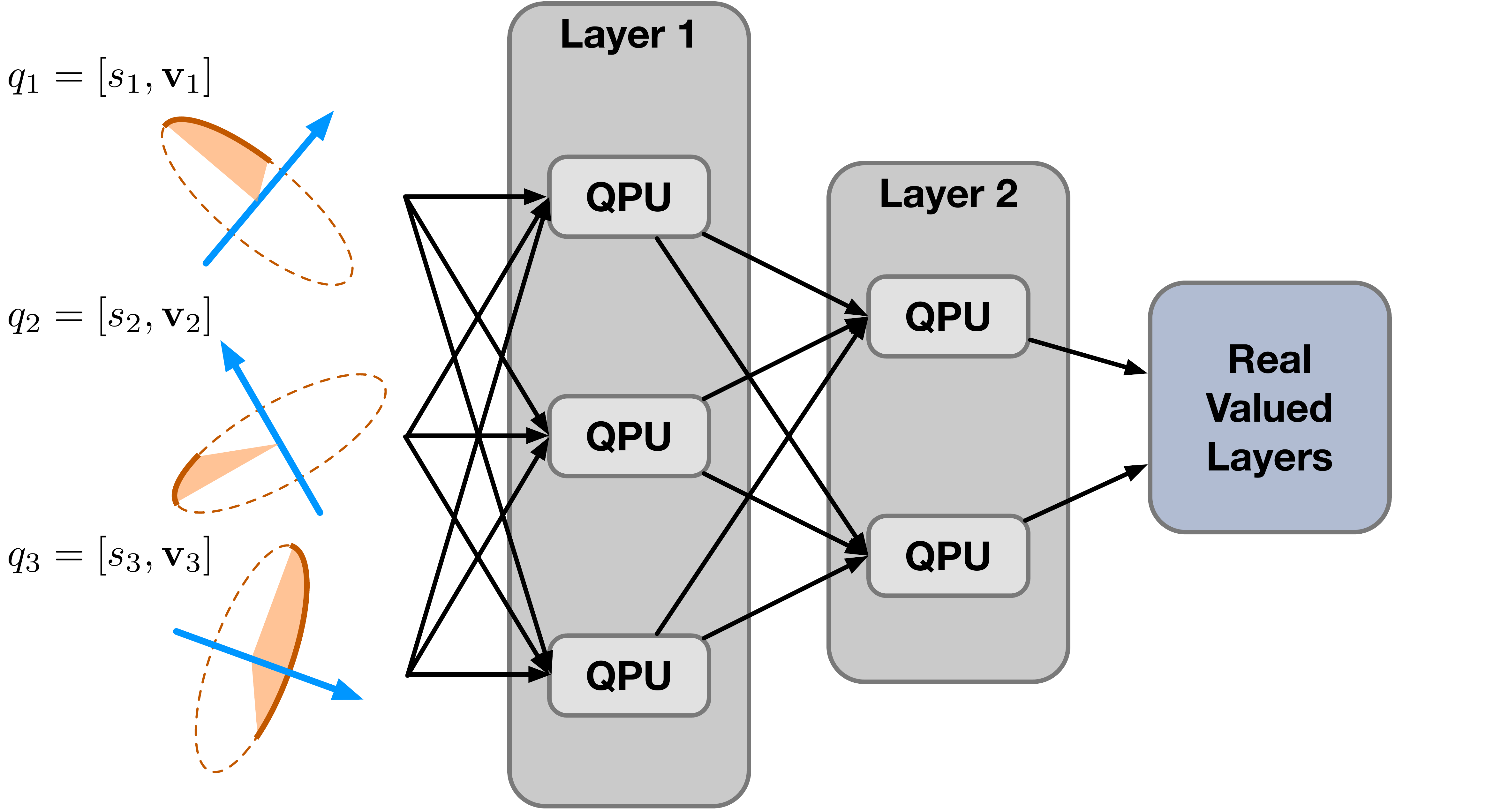}
        \caption{A QMLP with 2 QPU-based FC layers + real-valued layers}\label{fig:qmlp}
    \end{subfigure}
    \caption{(a) An illustration of our QPU model. 
    Each rotation in $\mathbb{SO}(3)$ can be represented via a unit quaternion in a hypercomplex space $\mathbb{H}$, whose real part and imaginary part indicates the rotation angle and the direction of the rotation axis, respectively.
    The QPU merges input rotations via a weighted chain of Hamilton products and derives a rotation as its output, whose real part and imaginary part yield to the rotation-invariance and the rotation-equivariance, respectively.
    (b) We can further construct a fully-connection (FC) layer using multiple QPUs. The stacking of these QPU-based FC layers leads to a quaternion multi-layer perceptron (QMLP) model, and this model can be followed by existing real-valued layers.}
\end{figure*}

Currently, many deep learning-based representation methods have made efforts to enhance their robustness to rotations~\cite{cohen2018spherical,esteves2018learning,weiler20183d,chen2019clusternet,thomas2018tensor}. 
However, the architectures of their models are tailored for the data in the Euclidean space, rather than 3D rotation data. 
In particular, the 3D rotation data are not closed under the algebraic operations used in their models ($i.e.$, additions and multiplications).\footnote{For example, adding two 3D rotation matrices together or multiplying a scalar with a 3D rotation matrix will not result in a valid rotation matrix.} 
The mismatching between data and model makes these methods difficult to analyze the influence of rotations on their outputs quantitatively. 
Although this mismatching problem can be mitigated by augmenting training data with additional rotations~\cite{qi2017pointnet}, this solution gives us no theoretical guarantee.

To overcome the challenge mentioned above, we proposed a novel quaternion product unit (QPU) for 3D rotation data, which establishes an efficient and interpretable mechanism to enhance the robustness of deep learning models to rotations. 
As illustrated in Figure~\ref{fig:qpu}, for each 3D rotation, we represent it as a unit quaternion, whose imaginary part indicates the direction of its rotation axis, and the real part corresponds to the cosine of its rotation angle, respectively. 
Taking $N$ quaternions as its inputs, the QPU first applies quaternion power operation to each input, scaling their rotation angles and rotation axes, respectively. 
Then, it applies a chain of Hamilton products to merge the weighted quaternions and output a rotation accordingly. 
The parameters of the QPU consists of the scaling coefficients and the bias introduced to the rotation angles. 

The proposed QPU leverages quaternion algebra and the law of 3D rotation groups, merging the inputs as one quaternion. 
Because 3D rotation group is closed under the operations used in our QPU, the output is still a valid rotation. 
Moreover, given a dataset $\mathcal{X}\subset\mathbb{SO}(3)^N$, where each $\bm{x}=[x_1, ..., x_N]\in\mathcal{X}$ contains $N$ rotations, we can define two kinds of rotation robustness for a mapping function $f:~\mathbb{SO}(3)^N\mapsto\mathbb{SO}(3)$:

\textbf{Rotation-invariance} $f(R(\bm{x})) = f(\bm{x}),~\forall\bm{x}\in\mathcal{X}$.

\textbf{Rotation-equivariance} $f(R(\bm{x})) = R(f(\bm{x})),\forall\bm{x}\in\mathcal{X}$.

\noindent Here, $R$ is an arbitrary rotation operation and $R(\bm{x})=[R(x_1),...,R(x_N)]$.
We prove that the quaternion derived by QPU is disentangled naturally into a \emph{rotation-invariant} real part and a \emph{rotation-equivariant} imaginary part. 

The proposed QPU is highly compatible with existing deep learning models. 
As shown in Figure~\ref{fig:qmlp}, we build fully-connected (FC) layers based on QPUs and stack them as a quaternion multi-layer perceptron model (QMLP). 
We also design a quaternion graph convolutional layer based on QPUs. 
These QPU-based layers can be combined with standard real-valued layers. 
Experiments on both synthetic and real-world data show that the proposed QPU-based models can be trained efficiently with standard back-propagation and outperform traditional real-valued models in many applications requiring rotation robustness.

%% file: tex/QPU.tex
\section{Proposed Quaternion Product Units}
Many 3D data can be represented as points on 3D rotation groups.
Take 3D skeleton data as an example.
A 3D skeleton is a 3D point cloud associated with a graph structure.
We can represent a skeleton by a set of 3D rotations, using relative rotations between edges ($i.e.$, joint rotations). 
Mathematically, the 3D rotation from a vector $\vv_1 \in \mathbb{R}^3$ to another vector $\vv_2 \in \mathbb{R}^3$ could be described as a rotation around an axis $\uu$ with a angle $\theta$:
\begin{eqnarray}\label{eq:aa}
\begin{aligned}
&\text{Rotation Axis:}~\uu = \frac{\vv_1 \times \vv_2}{\|\vv_1 \times \vv_2\|_2},\\
&\text{Rotation Angle:}~\theta = \arccos(\frac{\langle \vv_1, \vv_2 \rangle}{\|\vv_1\|_2\|\vv_2\|_2}).
\end{aligned}
\end{eqnarray}
Here, $\vv_1 \times \vv_2$ stands for cross product and $\langle \vv_1, \vv_2\rangle$ is for inner product. 
Given such 3D rotation data, we would like to design a representation method yielding to a certain kind of rotation robustness. 
With the help of quaternion algebra, we propose a quaternion product unit to achieve this aim.

\subsection{Quaternion and 3D rotation}
Quaternion $q = s + \ii x + \jj y + \kk z \in \mathbb{H}$ is a type of hypercomplex number with 1D real part $s$ and 3D imaginary part $(x, y, z)$.
The imaginary parts satisfy $\ii ^2 = \jj ^2 = \kk ^2 = \ii \jj \kk = -1, \ii \jj = -\jj \ii = \kk, \kk \ii = -\ii \kk = \jj, \jj \kk = -\kk \jj = \ii$. 
For convenience, we ignore the imaginary symbols and represent a quaternion as the combination of a scalar and a vector as $q = [s, \vv] = [s, (x, y, z)]$. 
When equipped with Hamilton product as multiplication, noted as $\otimes$, and standard vector addition, the quaternion forms an algebra.
In particular, the Hamilton product between two quaternions $q_1 = [s_1, \vv_1]$ and $q_2 = [s_2, \vv_2]$ is defined as
\begin{eqnarray}\label{eq:qmul_v}
\begin{aligned}
q_1 \otimes q_2 &= [s_1 s_2 -  \langle \vv_1, \vv_2 \rangle, \vv_1 \times \vv_2 + s_1 \vv_2 + s_2 \vv_1]\\
    &= M_L(q_1) q_2 = M_R(q_2) q_1,
\end{aligned}
\end{eqnarray}
where for a quaternion $q=[s, (x, y, z)]$, we have
\begin{eqnarray}
M_L(q) = \left[ \begin{smallmatrix}s & -x & -y & -z \\ x & s & -z & y \\ y & z & s & -x \\ z & -y & x & s \end{smallmatrix} \right],
    M_R(q) = \left[ \begin{smallmatrix}s & -x & -y & -z \\ x & s & z & -y \\ y & -z & s & x \\ z & y & -x & s \end{smallmatrix} \right].
\end{eqnarray}
The second row in Eq.~(\ref{eq:qmul_v}) uses matrix-vector multiplications, in which we treat one quaternion as a 4D vector and compute a matrix based on the other quaternion.
Note that the Hamilton product is non-commutative --- multiplying a quaternion on the left or the right gives different matrices.

Quaternion algebra provides us with a new representation of 3D rotations. 
Specifically, suppose that we rotate a 3D vector $\vv_1$ to another 3D vector $\vv_2$, and the 3D rotation is with an axis $\uu$ and an angle $\theta$ shown in Eq.~(\ref{eq:aa}). 
We can represent the 3D rotation as a unit quaternion $q=[s,\vv]=[\cos(\frac{\theta}{2}), \sin(\frac{\theta}{2})\uu]$, where $\|\uu\|_2=1$ and $s^2 +\|\vv\|_2^2=\cos^2(\frac{\theta}{2})+\sin^2(\frac{\theta}{2})=1$. 
Representing the two 3D vectors as two pure quaternions, $i.e.$, $[0, \vv_1]$ and $[0,\vv_2]$, we can achieve the rotation from $\vv_1$ to $\vv_2$ via the Hamilton products of the corresponding quaternions:
\begin{equation}
[0, \vv_2] = q \otimes [0, \vv_1] \otimes q^{*},
\end{equation}
where $q^{*}=[s,-\vv]$ stands for the conjugation of $q$. 
Additionally, the combination of rotation matrices can also be represented by the Hamilton products of unit quaternions. 
For example, given two unit quaternions $q_1$ and $q_2$, which correspond to two rotations, $(q_2\otimes q_1)\otimes [0, \vv_1]\otimes (q_1^* \otimes q_2^*)$ means rotating $\vv_1$ sequentially through the two rotations.

Note that unit quaternion is a double cover of $\mathbb{SO}(3)$ since $q$ and $-q$ represents the same 3D rotation (in opposite directions). 
As shown in Eq.~(\ref{eq:aa}), we use inner product and $\arccos$ to generate unit quaternion and reduce this ambiguity by choosing the quaternion with positive real part. 
Refer to~\cite{dam1998quaternions} for more details on unit quaternion and 3D rotation.

\subsection{Quaternion product units}
In standard deep learning models, each of their neurons can be represented as a weighted summation unit, $i.e.$, $y =  \sigma(\sum_{i=1}^{N} w_i x_i + b)$, where $\{w_i\}_{i=1}^{N}$ and $b$ are learnable parameters, $\{x_i\}_{i=1}^{N}$ are inputs, and $\sigma(\cdot)$ is a nonlinear activation function. 
As aforementioned, when the input $x_i\in\mathbb{SO}(3)$, such a unit cannot keep the output $y$ on $\mathbb{SO}(3)$ as well. 
To design a computational unit guaranteeing the closure of $\mathbb{SO}(3)$, we propose an alternate of this unit based on the quaternion algebra introduced above.

Specifically, given $N$ unit quaternions $\{q_i=[s_i, \vv_i]\}_{i=1}^{N}$ that represent 3D rotations, we can define a weighted chain of Hamilton products as 
\begin{equation}\label{eq:chain}
    y = \sideset{}{_{i=1}^{N}}\bigotimes q_i ^ {w_i} = q_{1}^{w_1} \otimes q_{2}^{w_2} \otimes ... \otimes q_{N}^{w_N},
\end{equation}
where the power of a quaternion $q=[s,\vv]$ with a scalar $w$ is defined as~\cite{dam1998quaternions}:
\begin{eqnarray*}\label{qpow2}
\begin{aligned}
   q ^ w := 
   \begin{cases}
   [\cos(w \arccos(s)), \frac{\vv}{\|\vv\|_2} \sin(w \arccos(s))],\quad\vv\neq\bm{0},\\
   [1, \bm{0}],\quad\text{otherwise}.
   \end{cases}
\end{aligned}
\end{eqnarray*}
Note that the power of a quaternion only scales the rotation angle and does not change the rotation axis.

Here, we replace the weighted summation with a weighted chain of Hamilton products, which makes $\mathbb{SO}(3)$ closed under the operation. 
Based on the operation defined in Eq.~(\ref{eq:chain}), we proposed our quaternion product unit: 
\begin{eqnarray}\label{eq:qpu_eq}
\begin{aligned}
\text{QPU}(\{q_i\}_{i=1}^{N}; \{w_i\}_{i=1}^{N},b) =\sideset{}{_{i=1}^{N}}\bigotimes \text{qpow}(q_i; w_i, b),
\end{aligned}
\end{eqnarray}
where for $q_i=[s_i,\vv_i]$
\begin{eqnarray*}
\begin{aligned}
&\text{qpow}(q_i; w_i, b)\\
=&[\cos(w_i (\arccos(s_i) + b)), \frac{\vv_i}{\|\vv_i\|_2} \sin(w_i (\arccos(s_i) + b))]
\end{aligned}
\end{eqnarray*} 
represents the weighting function on the rotation angles with a weight $w_i$ and a bias $b$. 

Compared with Eq.~(\ref{eq:chain}), we add a bias to $b$ to shift the origin.
The output of $\arccos$ contains an infinite gradient at $\pm 1$.
We solve this problem by clamping the input scalar part $s$ between $-1+\epsilon$ and $1-\epsilon$, where $\epsilon$ is a small number.

\subsection{Rotation-invariance and equivariance of QPU}
The following proposition demonstrates the advantage of our QPU on achieving rotation robustness.
\begin{proposition}
The output of the QPU is a quaternion containing a rotation-invariant real part and a rotation-equivariant imaginary part.
\end{proposition}
\begin{proof}
This proposition is followed directly by the property of the Hamilton product. 
Given two quaternions $[s_1, \vv_1]$ and $[s_2, \vv_2]$, we apply a Hamilton product, $i.e.$, $[s_1, \vv_1]\otimes[s_2, \vv_1]=[s_o, \vv_o]$. 
Applying a rotation $R$ on the vector parts, we have
\begin{eqnarray}
\begin{aligned}
&[s_1, R(\vv_1)]\otimes[s_2, R(\vv_2)] \\
=&[s_1 s_2 - \langle R(\vv_1),R(\vv_2) \rangle,\\
&~R(\vv_1) \times R(\vv_2) + s_1 R(\vv_2) + s_2 R(\vv_1)].
\end{aligned}
\end{eqnarray}

Because $\langle R(\vv_1), R(\vv_2) \rangle = \langle \vv_1, \vv_2 \rangle$ and $R(\vv_1) \times R(\vv_2) = R(\vv_1 \times \vv_2)$, we have
\begin{eqnarray}
\begin{aligned}
&[s_1, R(\vv_1)]\otimes[s_2, R(\vv_2)]\\
=&[ s_1 s_2 - \langle \vv_1, \vv_2 \rangle,
    R(\vv_1 \times \vv_2 + s_1 \vv_2 + s_2 \vv_1)]\\
=&[s_o, R(\vv_o)].
\end{aligned}
\end{eqnarray}
Thus, the Hamilton product of two quaternions gives a rotation-invariant real part and a rotation-equivariant imaginary part. 
The same property holds for the chain of weighted Hamilton products. 
\end{proof}

The proof above indicates that the intrinsic property and the group law of $\mathbb{SO}(3)$ naturally provided rotation-invariance and rotation-equivariance. 
Without forced tricks or hand-crafted representation strategies, the principled design of QPU makes it flexible for a wide range of applications with different requirements on rotation robustness. 
Figure~\ref{fig:illu} further visualizes the property of our QPU. 
Given a QPU, we feed it with a human skeleton and its rotated version, respectively. 
We compare their outputs and find that the imaginary parts of their outputs inherit the rotation discrepancy between them while the real parts of their outputs are the same with each other.

\begin{figure}[t]
    \centering
    \begin{minipage}[b]{0.22\linewidth}
        \centering
        \begin{subfigure}[b]{0.84\linewidth}
            \includegraphics[height=1.5cm]{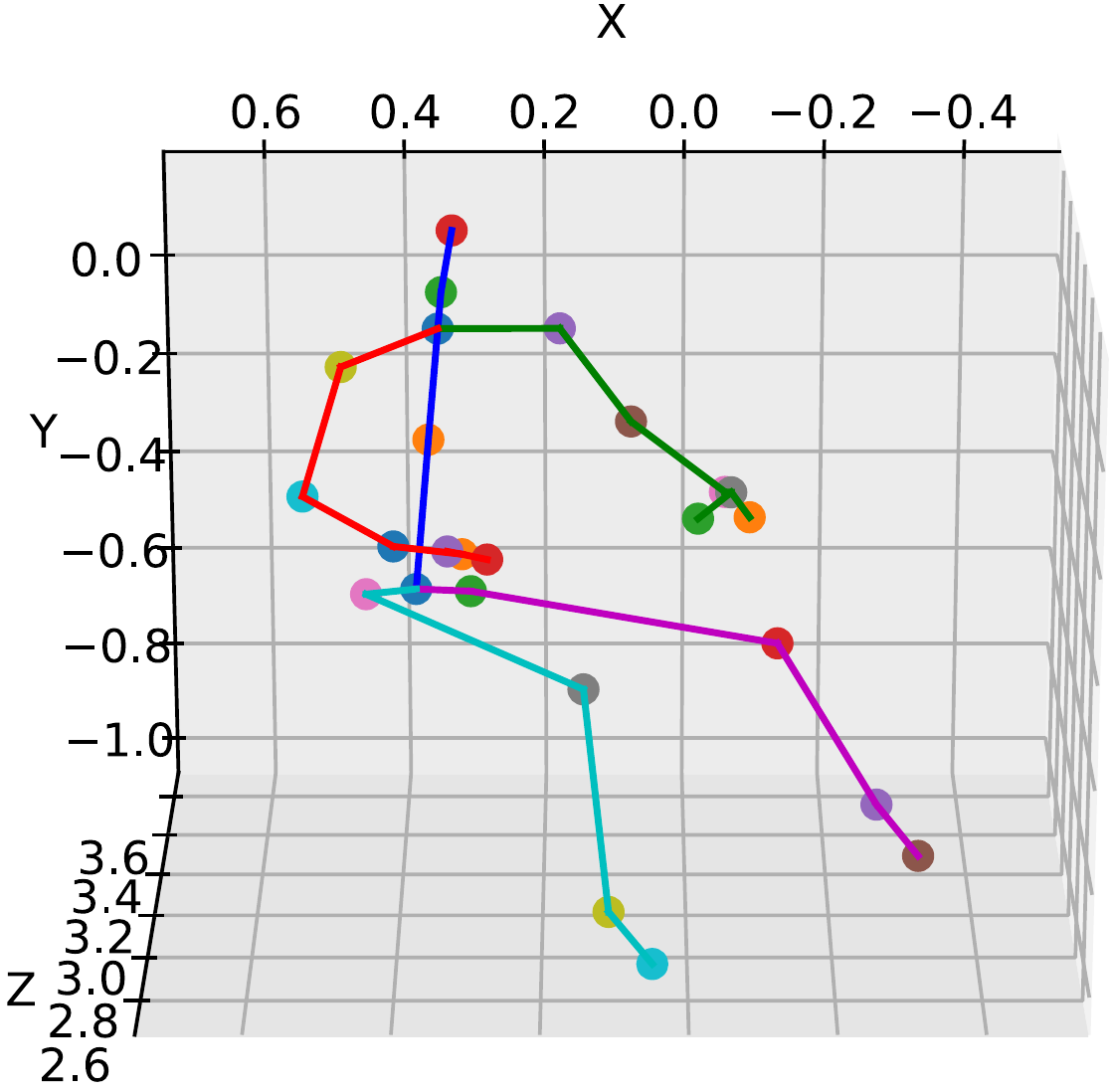}
            \caption{Original}\label{fig:in}
        \end{subfigure}
        \begin{subfigure}[b]{0.84\linewidth}
            \includegraphics[height=1.5cm]{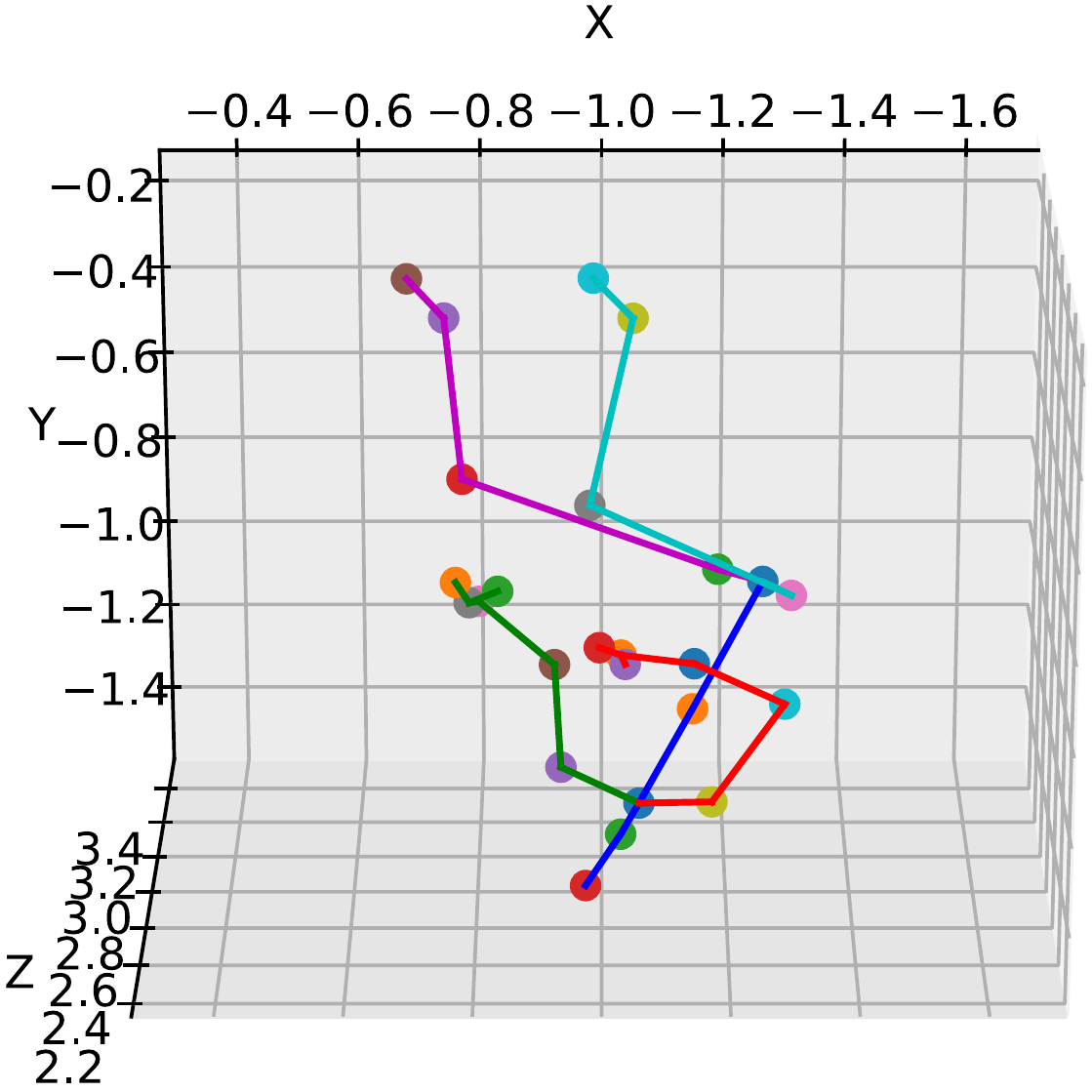}
            \caption{Rotated}\label{fig:rot}
        \end{subfigure}
    \end{minipage}
    \begin{minipage}[b]{0.22\linewidth}
        \centering
        \begin{subfigure}[b]{0.84\linewidth}
            \includegraphics[height=1.5cm]{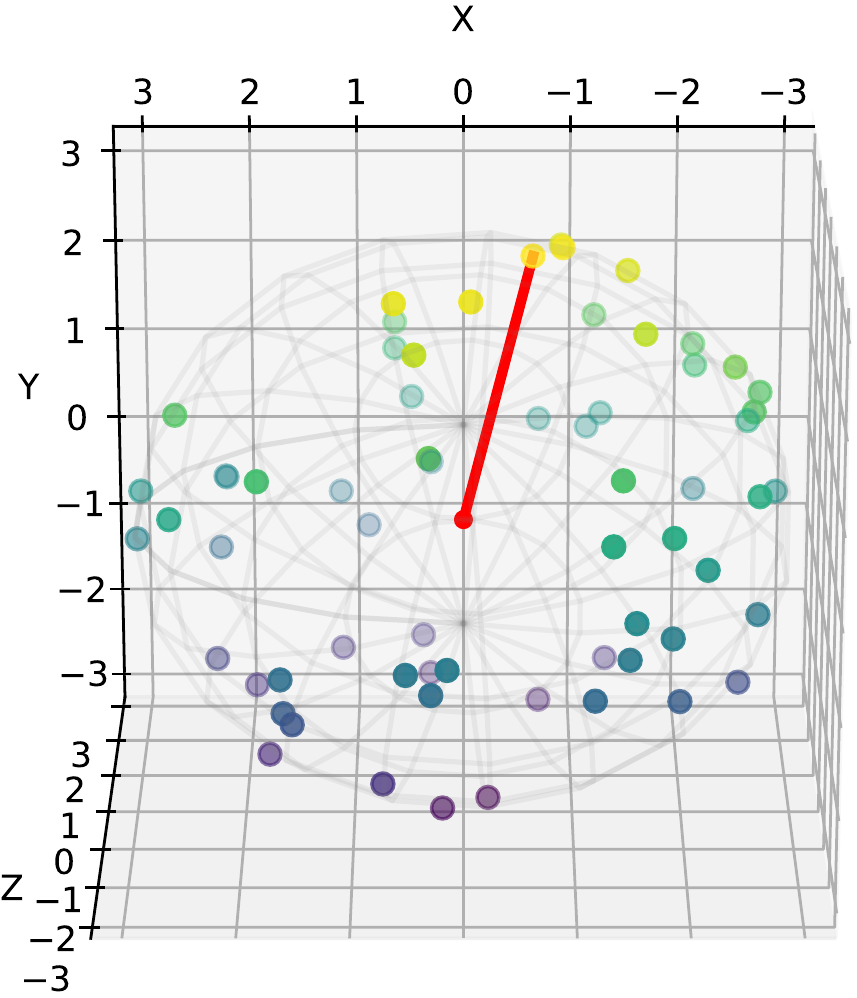}
            \caption{$\vv_o$'s}\label{fig:im1}
        \end{subfigure}
        \begin{subfigure}[b]{0.84\linewidth}
            \includegraphics[height=1.5cm]{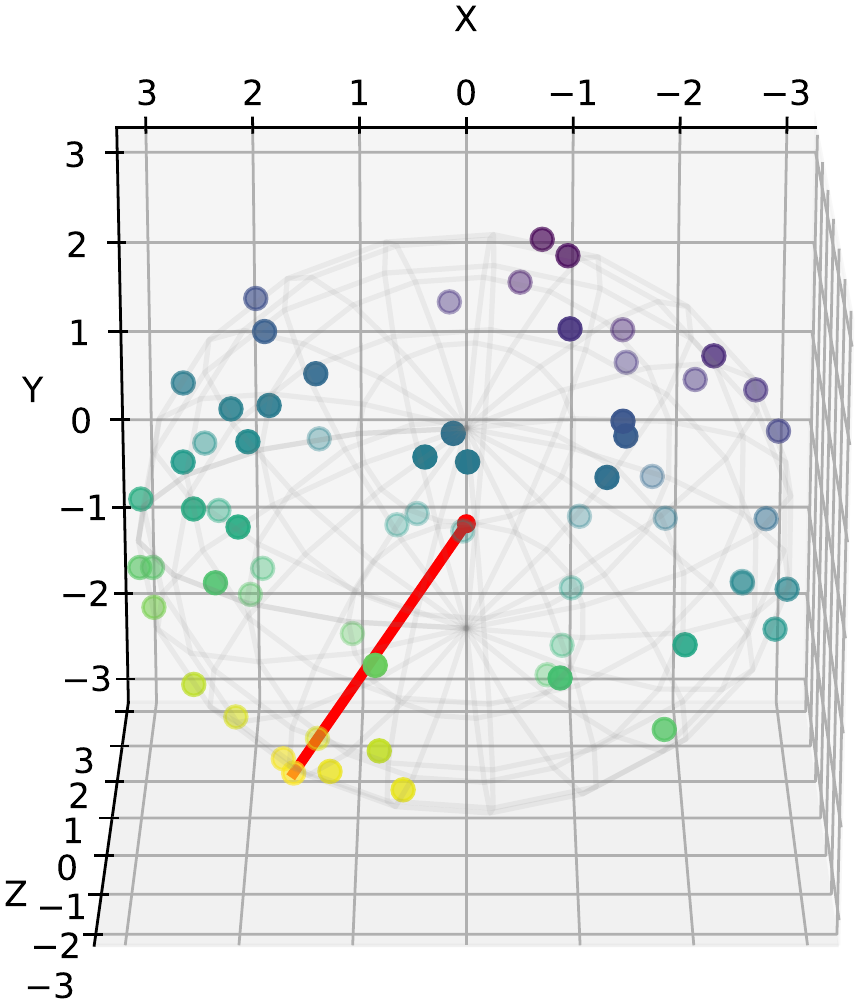}
            \caption{$\vv_o$'s}\label{fig:im2}
        \end{subfigure}
    \end{minipage}
    \begin{subfigure}[b]{0.5\linewidth}
        \includegraphics[height=3cm]{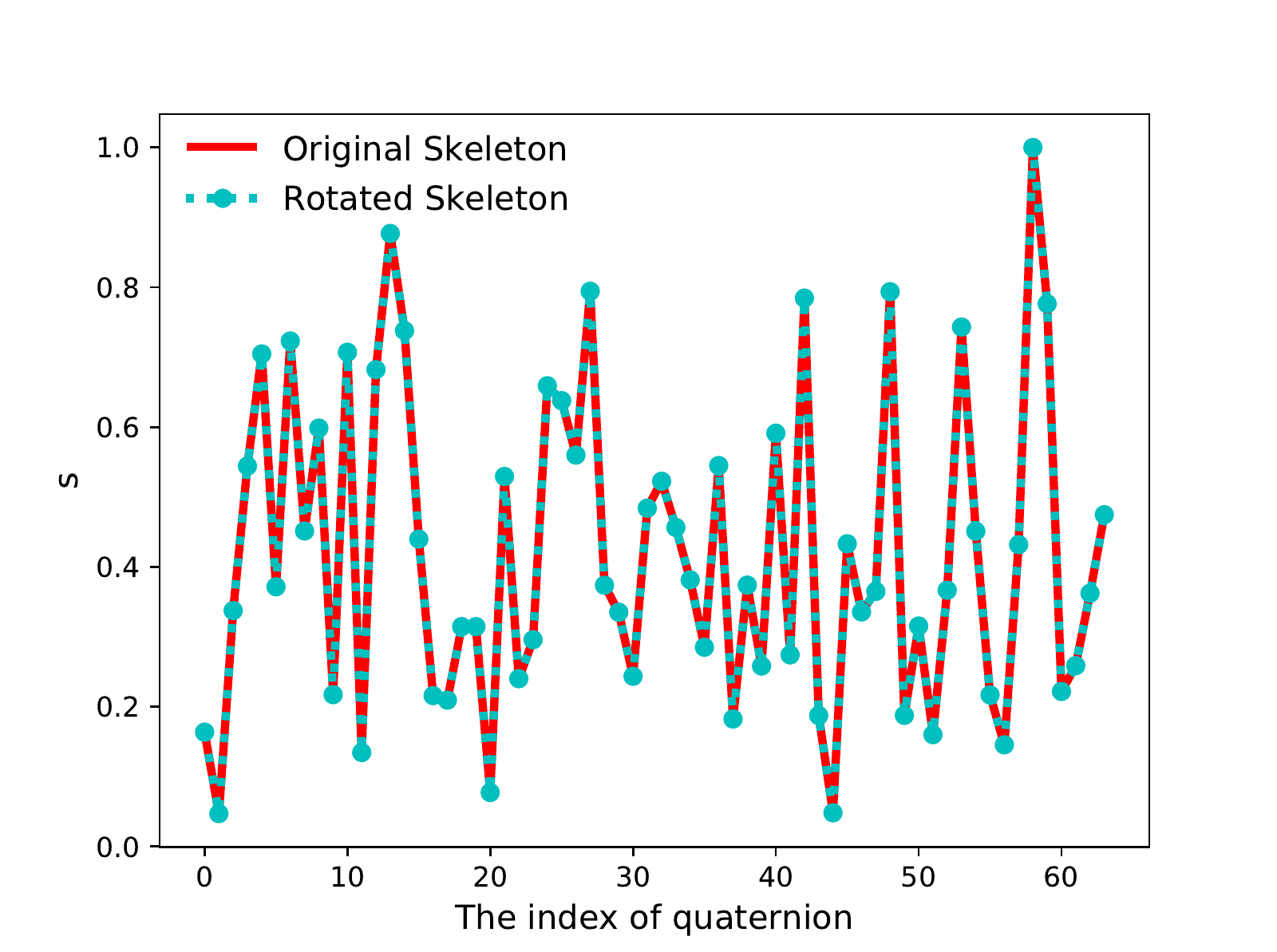}
        \caption{$s_o$'s}\label{fig:real}
    \end{subfigure}
    \caption{An illustration of the rotation-invariance and the rotation-equivariance of QPU. 
    For each skeleton, we represent the relative rotations between its joints as unit quaternions and treat these quaternions as the input of a QPU. 
    To visualize the rotation discrepancy, for a $\vv_o$ in (c) and the corresponding $\vv_o$ in (d), we connect each of them (normalized to a sphere) with the origin via a red line.}\label{fig:illu}
\end{figure}


\subsection{Backpropagation of QPU}\label{ssec:bp}
The forward pass of QPU involves $N$ individual weighting functions and a chain of Hamilton products.
While it is easy to compute the gradients of the weighting functions~\cite{dam1998quaternions}, the gradient of the chain of Hamilton products is not straightforward.

Given quaternions $\{q_i\}_{i=1}^{N}$, we first compute the differential of the $k$-th quaternion $q_k$ with respect to chain of Hamilton product $\bigotimes_{i=1}^{N} q_i$, denoted as $d(\bigotimes_{i} q_i)_{q_k} (\cdot)$, omitting $N$ for simplicity. 
As Hamilton product is bilinear, we have
\begin{eqnarray}
\begin{aligned}
d(\sideset{}{_i}\bigotimes q_i)_{q_k} (h)& = (\sideset{}{_{j<k}}\bigotimes q_j) \otimes h \otimes (\sideset{}{_{j>k}}\bigotimes q_j)\\
&=M_L(\sideset{}{_{j<k}}\bigotimes q_j) M_R({\sideset{}{_{j>k}}\bigotimes q_j}) h\\
&=M_L(B_k) M_R(A_k) h.
\end{aligned}
\end{eqnarray}
where $h$ represents a small variation of the $k$-th input quaternion, $B_k =\bigotimes_{j<k} q_j$ and $A_k = \bigotimes_{j>k} q_j$ are the chains of Hamilton product of the quaternions before and after  the $k$-th quaternion, respectively.

We then compute the differential of loss $L$ for the $k$-th quaternion.
Given the gradient of the scalar loss $L$ with respect to the output quaternion $\frac{\partial L}{\partial (\bigotimes_i q_i)}$ , \eg, computed with autograd~\cite{paszke2017automatic} of Pytorch, we have
\begin{eqnarray}
\begin{aligned}
dL_{q_k}(h) &= \langle \frac{\partial L}{\partial ( \bigotimes_i q_i)}, d(\sideset{}{_i} \bigotimes q_i)_{q_k}(h) \rangle\\
            &= \langle \frac{\partial L}{\partial ( \bigotimes_i q_i)}, M_L(B_k) M_R(A_k) h \rangle \\
            &= \langle M_R^T(A_k) M_L^T(B_k) \frac{\partial L}{\partial (\bigotimes_i q_i)}, h \rangle,
\end{aligned}
\end{eqnarray}
where $\langle \cdot, \cdot \rangle$ stands for the inner product between vectors and $M^T$ stands for the transpose of matrix $M$. 
Thus the gradient of loss $L$ for $q_k$ is given as
\begin{equation}
    \frac{\partial L}{\partial q_k} = M_R^T(A_k)  M_L^T(B_k) \frac{\partial L}{\partial ( \bigotimes_i q_i)}.
\end{equation}

To compute $B_k$ and $A_k$, we first compute the cumulative Hamilton product $C=(c_1, c_2, ..., c_N)$, where $c_k = \bigotimes_{j\leq k} q_j$.
Then $B_k = c_k \otimes q_k^*$ and $A_k = c_k^* \otimes c_N$.
We note that the gradient depends on a cumulative Hamilton product of all quaternions other than the $k$-th quaternion.
This computation process is different from that of standard weighted summation units, in which the backpropagation through addition only involves the differential variable instead of other inputs.
In contrast, it is interesting to see the gradient is a joint result of all other inputs to QPU except the items of the given differential variable.

\section{Implementations and Further Analysis}
\subsection{QPU-based neural networks}\label{ssec:qfc}
As shown in Figure~\ref{fig:qmlp}, multiple QPUs receiving the same input quaternions form a fully-connected (FC) layer. 
The parameters of the QPU-based FC layer include the weights and the bias used in the QPUs.
We initialize them using the Xavier uniform initialization~\cite{glorot2010understanding}.
Different from real-valued weighted summation, the QPU itself is nonlinear. 
Hence there is no need to introduce additional nonlinear activation functions. 
As aforementioned, the output of QPU is still a unit quaternion so that we can connect multiple QPU-based FC layers sequentially. 
Accordingly, the stack of multiple QPU-based FC layers establishes a quaternion multi-layer perceptron (QMLP) model for 3D rotation data. 
Note that  the Hamilton product is not commutative, $i.e.$, $q_1 \otimes q_2 \neq q_2 \otimes q_1$ in general, the stacking of two QPU-based FC layers is not equivalent to one QPU-based FC layer. 

Besides the QPU-based FC layer and the corresponding QMLP model, our QPU can also be used to implement a quaternion-based graph convolution layer. 
Specifically, given a graph with an adjacency matrix $A=[a_{ij}]\in\mathbb{R}^{N\times N}$, the aggregation of its node embeddings is achieved by the multiplication between its adjacency matrix and the embeddings~\cite{kipf2016semi}. 
When the node embeddings correspond to 3D rotations and are represented as unit quaternions, we can implement a new aggregation layer using the chain of Hamilton products shown in Eq.~(\ref{eq:chain}), $i.e.$, given input quaternions $\{q_i\}_{i=1}^{N}$ and the adjacency matrix $A$, the $i$-th output $q_{i}^{agg}=\bigotimes_{j=1}^{N} q_j^{a_{ij}}$.
Stacking such a QPU-based aggregation layer with a QPU-based FC layer, we achieve a quaternion-based graph convolution layer accordingly.

\subsection{Compatibility with real-valued models}\label{ssec:qr}
Besides building pure quaternion-valued models ($i.e.$, the QMLP mentioned above), we can plug our QPU-based layers into existing real-valued models easily.
Suppose that we have one QPU-based layer, which takes $N$ unit quaternions as inputs and derives $M$ unit quaternions accordingly. 
When the QPU-based layer receives rotations from a real-valued model, we merely need to reformulate the inputs as unit quaternions.
When a real-valued layer follows the QPU-based layer, we can treat the output of the QPU-based layer as a real-valued matrix with size $M\times 4$ and directly feed the output to the subsequent real-valued layer. 
Moreover, when rotation-invariant (rotation-equivariant) features are particularly required, we can feed only the real part (the imaginary part) to the real-valued layer accordingly.

Besides the straightforward strategy above, we further propose an angle-axis map to connect QPU-based layers with real-valued ones. 
Specifically, given a unit quaternion $q=[s,\vv]$, the proposed mapping function is defined as: 
\begin{equation}\label{angle-axis}
\text{AngleAxisMap}(q) = [\arccos(s), \frac{\vv}{\|\vv\|_2}].
\end{equation}
We use this function to map the output of a QPU, which lies on a non-Euclidean manifold~\cite{huynh2009metrics}, to the Euclidean space. 
Using this function to connect QPU-based layers with real-valued ones makes the learning of the downstream real-valued models efficient, which helps us improve learning results. 
Our angle-axis mapping function is better than the logarithmic map used in~\cite{vemulapalli2016rolling, vemulapalli2014human,huang2017deep,szczkesna2018quaternion}. 
The logarithmic map $\log(q) = [0, \arccos(s)\frac{\vv}{\|\vv\|_2}]$ mixes the angular and axial information, while our map keeps this disentanglement,  which is essential to preserve the rotation-invariance and rotation-equivariance of the output.

\subsection{Computational Complexity and Accelerations}
As shown in Section~\ref{ssec:bp}, our QPU supports gradient-based update, so both QPU-based layers and the models combining QPU-based layers with real-valued ones can be trained using backpropagation. 
Since each QPU handles a quaternion as one single element, QPU is efficient in terms of parameter numbers.
A real-valued FC with $N$ inputs and $M$ outputs requires $NM$ parameters, while a QPU with the same input and output dimensions requires only $\frac{1}{16}NM$ parameters. 

Due to the usage of Hamilton products, however, QPU requires more computational time than the real-valued weighted summation unit. 
Take a QPU-based FC layer with $N$ input quaternions and $M$ output quaternions as an example.
In the forward pass, the weighting step requires $N M$ multiplications and $N M$ sine and cosine computations.
Each Hamilton product requires $16$ multiplications and $12$ additions.
So the chain of $N-1$ Hamilton products requires $16 (N-1) M$ multiplications and $12 (N-1) M$ additions.
In total, a QPU-based FC layer requires $(17 N - 16) M$ multiplications, $(13 N - 12) M$ additions and $N M$ cosine and sine computations.
As a comparison, the real-valued FC layer with the same input and output size ($i.e.$, $4N$ inputs and $4M$ outputs in terms of real numbers) requires $16 NM$ multiplications, $16(N-1) M$ additions. 

In the backward pass, the gradient computation of a chain of Hamilton products requires the cumulative Hamilton product of the power weighted input quaternions. 
In particular, when computing the gradient of a chain of Hamilton products for an input quaternion, we need to do two Hamilton products and one matrix-matrix multiplication. 
As a result, the computational time is at least doubled compared with the forward pass and tripled if we recompute the cumulative Hamilton product. 
To deal with this challenge, we have two strategies: (a) store the result of cumulative Hamilton product in the forward pass, (b) recompute the same quantity during the backward pass.
The strategy (a) saves computational time but requires to save a potentially large feature map ($i.e.$, $N$ times the size of the feature map). 
The strategy (b) requires more computation time but no additional memory space.
We tested both options in our experiments.
When the QPU-based layer is with small feature maps ($i.e.$, at the beginning of a neural network with fewer input channels), both these two strategies work well. 
When we stack multiple QPU-based layers and apply the model to large feature maps, we need powerful GPUs with large memory spaces to implement the first strategy.

According to the above analysis, the computational bottleneck of our QPU is the chain of Hamilton products. 
Fortunately, we can accelerate this step by breaking down the computation like a tree. 
Specifically, as Hamilton product satisfies the combination law, $i.e.$, $(q_1 \otimes q_2) \otimes q_3 = q_1 \otimes (q_2 \otimes q_3)$, we can multiply quaternions at odd positions with quaternions at even positions in a parallel way and compute the whole chain by repeating this step recursively. For a chain of $N-1$ Hamilton products, this parallel strategy reduces the time complexity from $\mathcal{O}(N)$ to $\mathcal{O}(\log(N))$.
We implemented this strategy in the forward pass and witnessed a significant acceleration. 
In summary, although the backpropagation of our QPU-based model requires much more computational resources ($i.e.$, time, or memory) than that of real-valued models, its forward pass merely requires slightly more computations than that of real-valued models and can be accelerated efficiently via a simple parallel strategy.

%% file: tex/related_works.tex
\section{Related Work}
\paragraph{3D deep learning}
Deep learning has been widely used in the tasks relevant to 3D data. 
A typical application is skeleton-based action recognition~\cite{yan2018spatial,si2019attention,shi2019skeleton,huang2017deep}. 
Most existing methods often use 3D positions of joints as their inputs. 
They explore the graph structure of the skeleton data and achieve significant gains in performance.
For example, the AGC-LSTM in~\cite{si2019attention} introduces graph convolution to an LSTM layer and enhances the performance with an attention mechanism.
The DGNN in~\cite{shi2019skeleton} uses both node feature and edge feature to perform deep learning on a directed graph. 
In \cite{nguyen2019neural}, deep learning on symmetric positive definite (SPD) matrix was used to perform skeleton-based hand gesture recognition. 
Recently, several works are proposed to perform action recognition based on 3D rotations between bones.
Performing human action recognition by representing human skeletons in a Lie group (rotation and translation) is first proposed in~\cite{vemulapalli2014human} and further explored in~\cite{vemulapalli2016rolling}.
The LieNet in~\cite{huang2017deep} utilizes similar data representation in a deep learning framework on 3D rotation manifold, where input rotation matrices were transformed progressively by being multiplied with learnable rotation matrices.
While the LieNet uses pooling to aggregate rotation features, our proposed QPU is a more general learning unit, where Hamilton products could explore the interaction between rotating features.
Moreover, the weights in the LieNet are rotation matrices, whose optimization is constrained on the Lie group, while our QPU has unconstrained real-valued weights and can be optimized using standard backpropagation.

Besides the methods focusing on 3D skeletons, many neural networks are designed for 3D point clouds. 
The spherical CNNs~\cite{cohen2018spherical,esteves2018learning} project a 3D signal onto a sphere and perform convolutions in the frequency domain using spherical harmonics, which achieve rotation-invariant features.
The ClusterNet~\cite{chen2019clusternet} transforms point cloud into a rotation-invariant representation as well. 
The 3D Steerable CNNs~\cite{weiler20183d} and the TensorField Networks~\cite{thomas2018tensor} incorporate group representation theory in 3D deep learning and learn rotation-equivariant features by computing equivariant kernel basis on fields.
Note that these models can only achieve either rotation-invariance or rotation-equivariance. To our knowledge, our QPU makes the first attempt to achieve these two important properties in a unified framework.

\paragraph{Quaternion-based Learning}
Quaternion is widely used in computer graphics and control theory to represent 3D rotation, which only requires four parameters to describe a rotation matrix.
Recently, many efforts have been made to introduce quaternion-based models to the applications of computer vision and machine learning. 
Quaternion wavelet transforms (QWT)~\cite{bayro2006theory,zhou2007quaternion} and quaternion sparse coding~\cite{yu2013quaternion} have been successfully applied for image processing.
For skeleton-based action recognition, the quaternion lifting schema~\cite{szczkesna2018quaternion} used the SLERP~\cite{shoemake1985animating} to extract hierarchical information in rotation data to perform gait recognition.
The QuaterNet~\cite{pavllo2018quaternet} performs human skeleton action prediction by predicting the relative rotation of each joint in the next step with unit quaternion representation.

More recently, quaternion neural networks (QNNs) have produced significant advances~\cite{zhu2018quaternion,parcollet2018quaternion,parcollet2019quaternion,tay2019lightweight}. 
The QNNs replace multiplications in standard networks with Hamilton products, which can better preserve the interrelationship between channels and reduce parameters.
The QCNN in~\cite{zhu2018quaternion} proposed to represent rotation in color space with quaternion.
The QRNN in~\cite{parcollet2018quaternion} extended RNN to a quaternion version and applied it to speech tasks. 
Different from these existing models, the proposed QPU roots its motivation in the way of interaction on the 3D rotation group and uses Hamilton product to merge power weighted quaternions.

%% file: tex/experiments.tex
\section{Experiments}
To demonstrate the effectiveness of our QPU, we design the QMLP model mentioned above and test it on both synthetic and real-world datasets. 
In particular, we focus on the 3D skeleton classification tasks like human action recognition and hand action recognition.
For our QPU-based models, we represent each skeleton as the relative rotations between its connected bones ($a.k.a.$, joint rotations). 
Taking these rotations as inputs, we train our models to classify the skeletons and compare them with state-of-the-art rotation representation methods.
Besides testing the pure QPU-based models, we also use QPU-based layers to replace real-valued layers in existing models~\cite{si2019attention, shi2019skeleton} and verify the compatibility of our model accordingly.
The code is available at \url{https://github.com/IICNELAB/qpu_code}.

\subsection{Synthetic dataset --- CubeEdge}
We first design a synthetic dataset called ``CubeEdge'' and propose a simple task to demonstrate the rotation robustness of our QPU-based model. 
The dataset consists of skeletons composed of consecutive edges taken from a 3D cube (a chain of edges).
These skeletons are categorized into classes according to their shapes, $i.e.$, the topology of consecutive edges.
For each skeleton, we use the rotations between consecutive edges as its feature.

For each skeleton in the dataset, we generate it via the following three steps, as shown in Figure~\ref{fig:cube_gen}:
$i$) We initialize the first two adjacent edges deterministically.
$ii$) We ensure that the consecutive edges cannot be the same edge, so given the ending vertex of the previous edge, we generate the next edge randomly along either of the two valid directions.  
Accordingly, each added edge doubles the number of possible shapes. 
We repeat this step several times to generate a skeleton.
$iii$) We apply random shear transformations to the skeletons and augment the dataset, and split the dataset for training and testing. 
We generate each skeleton with seven edges so that our dataset consists of 32 different shapes (classes) of skeletons. 
We use 2,000 samples for training and testing, respectively. 
For the skeletons in the testing set, we add Gaussian noise $\mathcal{N}(0, \sigma^2)$ to their vertices before shearing, where the standard deviation $\sigma$ controls the level of noise. 
Figure~\ref{fig:cube_sample} gives four testing samples.
Additional random rotation is applied to testing set when validating the robustness to rotations. 

\begin{figure}[t]
    \centering
        \begin{subfigure}[b]{0.84\linewidth}
            \includegraphics[width=1\linewidth]{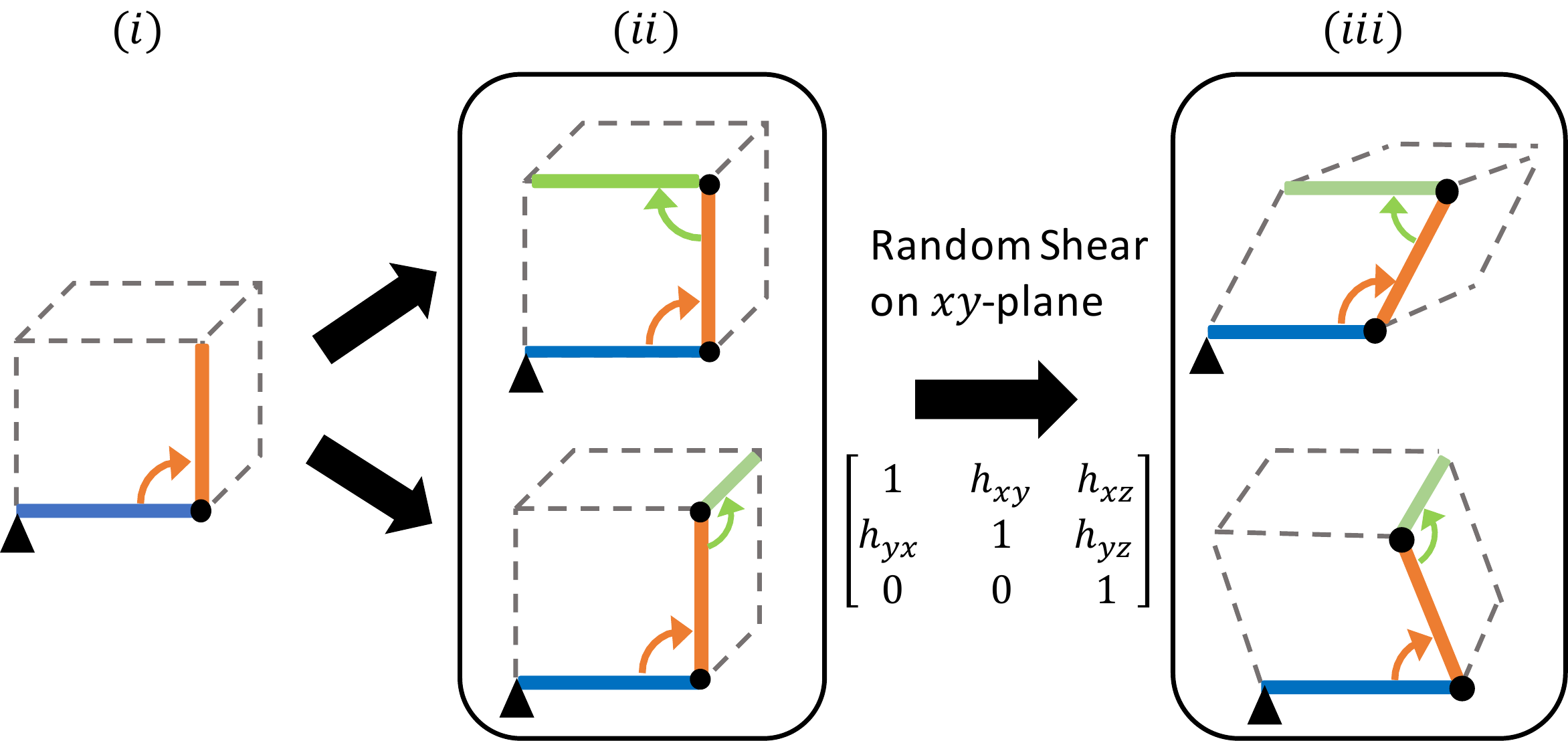}
            \caption{Generation process of CubeEdge. The root vertex is marked with a triangle}
            \label{fig:cube_gen}
        \end{subfigure}
        \begin{subfigure}[b]{0.84\linewidth}
            \includegraphics[width=1\linewidth]{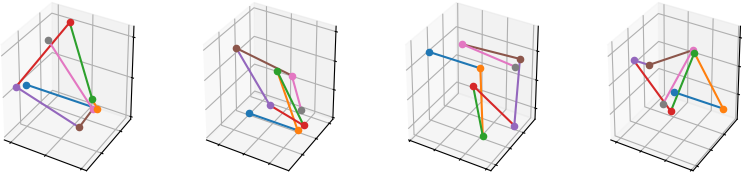}
            \caption{Samples from testing set of CubeEdge with 7 edges}
            \label{fig:cube_sample}
        \end{subfigure}
    \caption{Illustrations of the CubeEdge dataset. 
    (a) The steps to generate the skeletons. $i$) fix the two first edges; $ii$) pick the next edge from two possible candidates (doubling the number of classes); $iii$) apply random shear on $xy$-plane. (b) Samples from the testing set.
    }
    \label{fig:cube}
\end{figure}

\begin{figure}[t]
\centering        
        \begin{subfigure}[b]{0.29\linewidth}
            \includegraphics[width=1\linewidth]{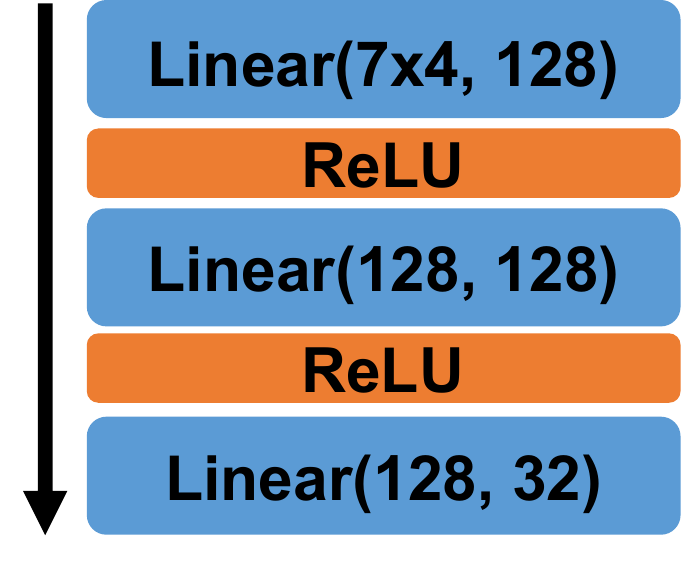}
            \caption{RMLP}
            \label{fig:toy_rmlp}
        \end{subfigure}
        \quad
        \begin{subfigure}[b]{0.29\linewidth}
            \includegraphics[width=1\linewidth]{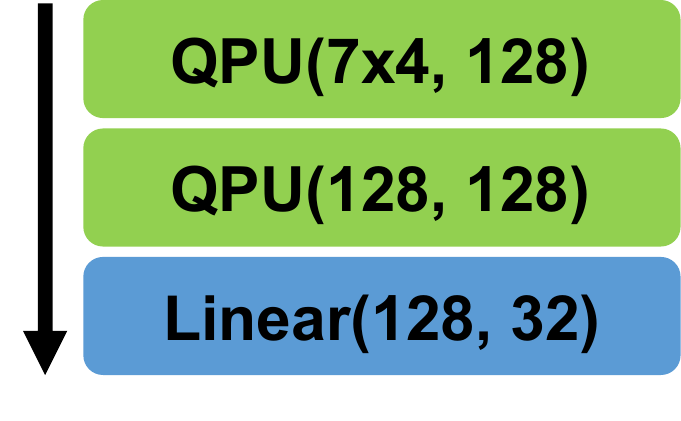}
            \caption{QMLP}
            \label{fig:toy_qmlp}
        \end{subfigure}
        \quad
        \begin{subfigure}[b]{0.29\linewidth}
            \includegraphics[width=1\linewidth]{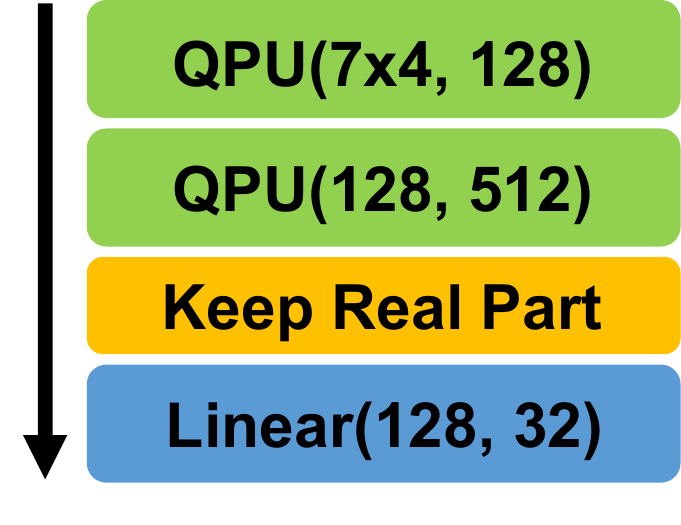}
            \caption{QMLP-RInv}
            \label{fig:toy_qmlp_rinv}
        \end{subfigure}
    \caption{The architectures of the models for synthetic data.}
    \label{fig:toy_models}
\end{figure}

To classify the synthetic skeletons, we apply our \textbf{QMLP} model and its variant \textbf{QMLP-RInv} and compared them with a standard real-valued MLP (\textbf{RMLP}). 
For fairness, all three models are implemented with three layers, and the size of output feature maps are the same.
The RMLP is composed of three fully-connected layers connected by two ReLU activation functions. 
Our QMLP is composed of two stacked QPU layers, followed by a fully-connected layer. 
Our QMLP-RInv is composed of two stacked QPU layers with only the real part of the output retained, and a fully-connected layer is followed.
The architectures of these three models are shown in Figure~\ref{fig:toy_models}.

After training these three models, we test them on the following two scenarios: $i$) comparing them on the testing set directly; and $ii$) adding random rotations to the testing skeletons and then comparing the models on the randomly-rotated testing set.
The first scenario aims at evaluating the feature extraction ability of different models, while the second one aims at validating their rotation-invariance.
For each scenario, we compare these three models on their classification accuracy. 
Their results with respect to different levels of noise ($i.e.$, $\sigma$) are shown in Figure~\ref{fig:toy}.
In both scenarios, our QMLP and QMLP-RInv outperform the RMLP consistently, which demonstrates the superiority of our QPU model. 
Furthermore, we find that the QMLP-RInv is robust to the random rotations imposed on the testing data --- it achieves nearly the same accuracy in both scenarios and works better than QMLP in the scenario with random rotations. 
This phenomenon verifies our claims: $i$) the real part of our QPU's output is rotation-invariant indeed; $ii$) the disentangled representation of rotation-invariant feature and rotation-equivariant feature could inherently make the network robust to rotations.

\begin{figure}[t]
\centering
        \begin{subfigure}[b]{0.49\linewidth}
            \includegraphics[width=1.0\linewidth]{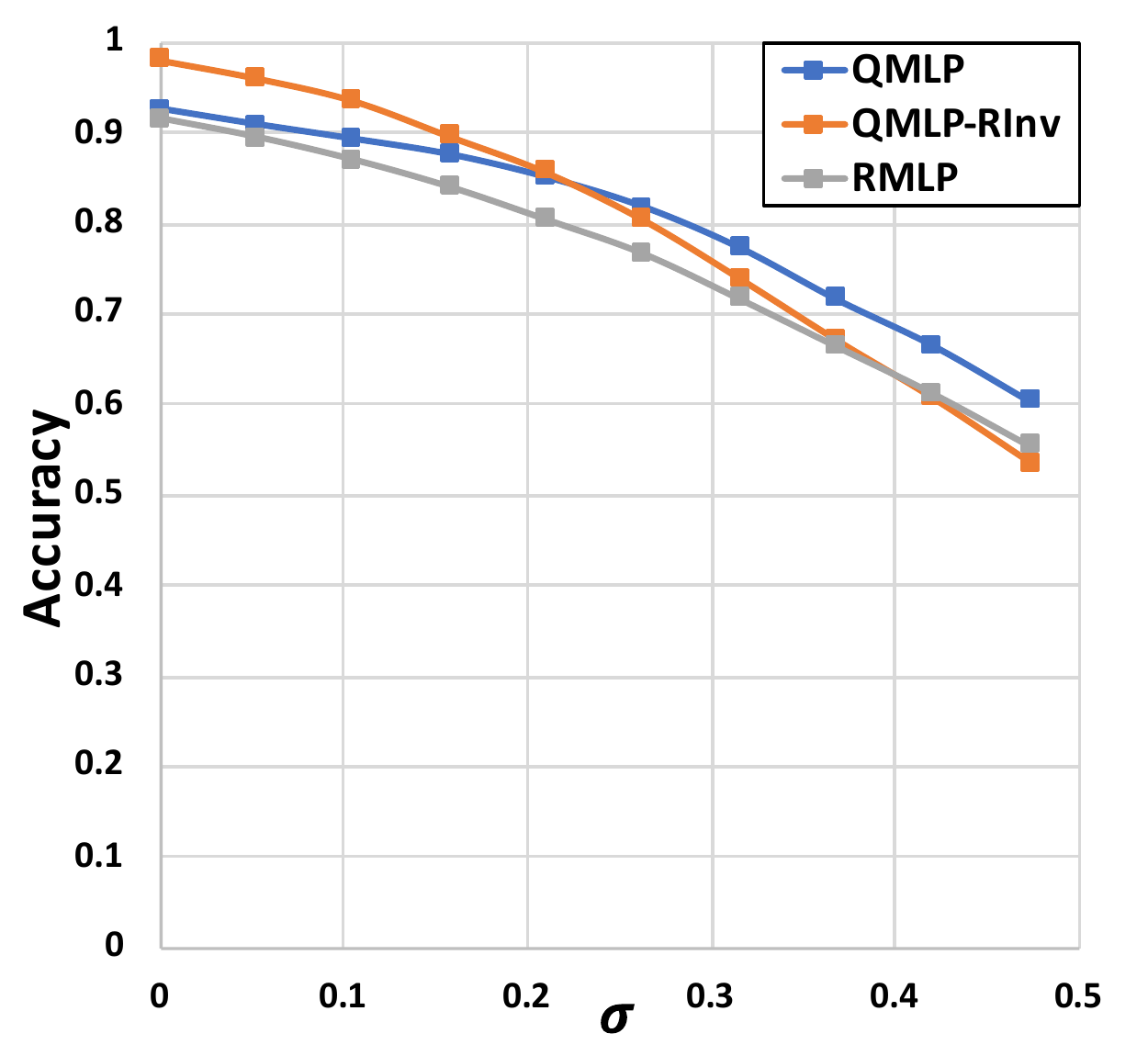}
            \caption{Without random rotations}
            \label{fig:toy6}
        \end{subfigure}
        \begin{subfigure}[b]{0.49\linewidth}
            \includegraphics[width=1.0\linewidth]{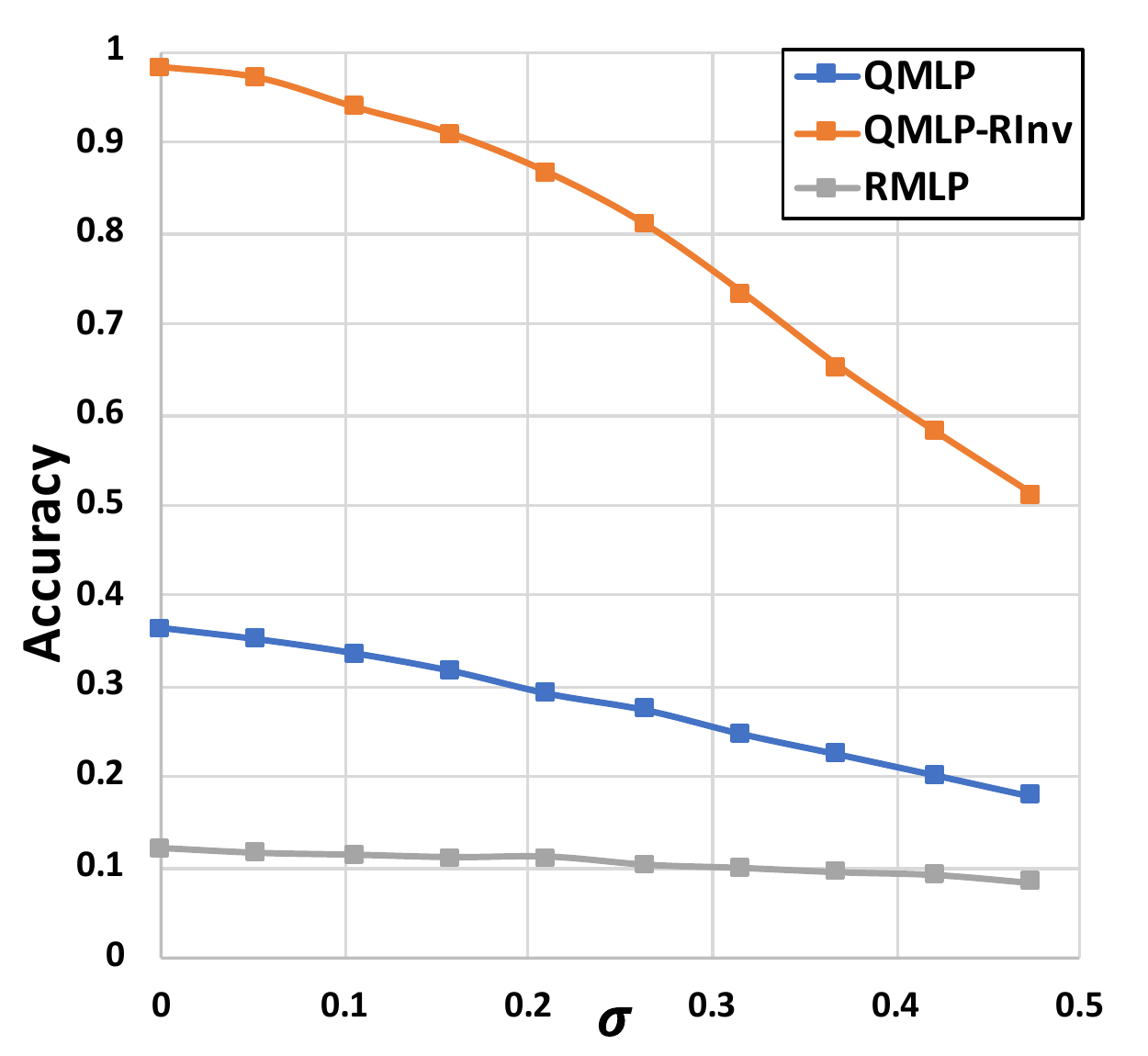}
            \caption{With random rotations}
            \label{fig:toy7}
        \end{subfigure}%
    \caption{Comparisons on testing accuracy.}
    \label{fig:toy}
\end{figure}

\subsection{Real-world data}
Besides testing on synthetic data, we also consider two real-world skeleton datasets: the \textbf{NTU} dataset~\cite{shahroudy2016ntu} for human action recognition and the \textbf{FPHA} dataset~\cite{garcia2018first} for hand action recognition.
The NTU dataset provides 56,578 human skeleton sequences performed by 40 actors belonging to 60 classes. 
The cross-view protocol is used in our experiments.
We use 37,646 sequences for training and the remaining 18,932 sequences for testing. 
FPHA dataset provides hand skeletons recorded with magnetic sensors and computed by inverse kinematics, which consists of 1,175 action videos belonging to 45 categories and performed by 6 actors. 
Following the setting in~\cite{garcia2018first}, we use 600 sequences for training and 575 sequences for testing.

For these two datasets, we implement three real-valued models as our baselines.
    
    \textbf{RMLP-LSTM} is composed of a two-layer MLP and a one-layer LSTM. 
    The MLP merges each frame into a feature vector and is shared between frames. 
    Features of each frame are then fed into the LSTM layer. 
    We average the outputs of the LSTM at different steps and feed it into a classifier.
    
    \textbf{AGC-LSTM}~\cite{si2019attention} first uses a FC layer on each joint to increase the feature dimension. 
    The features are fed into a 3-layer graph convolutional network. 
    
    \textbf{DGNN}~\cite{shi2019skeleton} is composed of 10 directed graph network (DGN) blocks that merge both node features and edge features. 
    Temporal information is extracted using temporal convolution between the DGN blocks.

To demonstrate the usefulness of our QPU, we substitute some of their layers with our QPU-based FC layers, and propose the following three QPU-based models:
    
    \textbf{QMLP-LSTM}:
     The two FC layers of the RMLP-LSTM are replaced with two QPU-based FC layers.
     
    
    \textbf{QAGC-LSTM}: The input FC layer of the AGC-LSTM replaced with a QPU-based FC layer. 
    In the original AGC-LSTM, the first FC layer only receives a feature from one joint. 
    As a result, the input would be a single quaternion, and the QPU layer would be unable to capture interactions between inputs.
    We thus augment the input by stacking the feature of the parent joint so that the input channel to QPU is two quaternions.

    
    \textbf{QDGNN}: The FC layers in the first DGN block of QDGNN are replaced with QPU-based FC layers. 
    The original DGNN uses 3D coordinates of joints as node features and vectors from parent joint to child joint as edge features.
    In our QDGNN, the node features are the unit quaternions corresponding to the joint rotations and the edge features are the differences between joint rotations, $i.e.$, $q_1 \otimes q_{2}^{*}$ computes the edge features for joint rotations $q_1$ and $q_2$. 
    Moreover, the original DGNN aggregates node features and edge features through the normalized adjacent matrix. 
    In our QDGNN, we implement this aggregation as the QPU-based aggregation layer proposed in Section~\ref{ssec:qfc}.

Similar to the experiments on synthetic data, we consider for the three QPU-based models their variants that only keep the real part of each QPU's output, denoted as \textbf{QMLP-LSTM-RInv}, \textbf{QAGC-LSTM-RInv}, and \textbf{QDGNN-RInv}, respectively.
For these variants, we multiply the output channels of the last QPU-based FC layer by $4$ and keep only the real parts of its outputs. 
In all real-world experiments, we use the angle-axis map (Eq.~(\ref{angle-axis})) to connect the QPU-based layers and real-valued neural networks.
The number of parameters of the QPU-based models is almost the same with that of the corresponding real-valued models.

We compare our QPU-based models with corresponding baselines under two configurations: $i$) both training data and test data have no additional rotations (NR); $ii$) the training data are without additional rotations while the testing ones are with arbitrary rotations (AR).
The comparison results on NTU and FPHA are shown in Table~\ref{tab:result}, respectively. 
We can find that in the first configuration, the performance of our QPU-based models is at least comparable to the baselines. 
In the second and more challenging setting, which requires rotation-invariance, our QPU-based models that use only outputs' real parts as features retain high accuracy consistently in most situations, while the performance of the baselines degrades a lot. 
Additionally, we quantitatively analyze the impact of our angle-axis map on testing accuracy.
In Table~\ref{tab:2}, the results on the FPHA dataset verify our claim in Section~\ref{ssec:qr}: for our QPU-based models, their performance boosts a lot when we connect their QPU-based layers with the following real-valued layers through the angle-axis map.

\begin{table}[t]
  \caption{Comparisons on testing accuracy (\%).}\label{tab:result}
  \vspace{-6pt}
  \centering
  \begin{tabular}{@{\hspace{2pt}}l@{\hspace{3pt}}c@{\hspace{3pt}}|
  @{\hspace{3pt}}c@{\hspace{6pt}}c@{\hspace{3pt}}|
  @{\hspace{3pt}}c@{\hspace{6pt}}c@{\hspace{2pt}}}\hline\hline
    \multirow{2}{*}{Model} &\#Param. &\multicolumn{2}{c|@{\hspace{3pt}}}{NTU}  & \multicolumn{2}{@{\hspace{3pt}}c}{FPHA} \\
    &(Million) &NR &AR &NR &AR\\\hline
    RMLP-LSTM       & 0.691 & 72.92 & 24.67 & 67.13 & 17.39\\
    QMLP-LSTM       & 0.609 & 69.72 & 26.60 & \textbf{76.17} & 24.00 \\
    QMLP-LSTM-RInv  & 0.621 & \textbf{75.84} & \textbf{75.84} & 68.00 & \textbf{67.83} \\\hline
    AGC-LSTM        & 23.371 & \textbf{90.50}  & 55.17 & 77.22  & 24.70 \\
    QAGC-LSTM       & 23.368 & 87.18 & 39.43 & \textbf{79.13} & 24.70 \\
    QAGC-LSTM-RInv  &23.369 & 89.68 & \textbf{89.92} & 72.35 &\textbf{71.30}\\\hline
    DGNN            & 4.078 & \textbf{87.45} & 23.30 & 80.35 & 24.35\\
    QDGNN           & 4.075 & 86.55 & 41.06 & \textbf{82.26} & 27.13\\
    QDGNN-RInv      & 4.076 & 83.88 & \textbf{83.88} & 76.35 & \textbf{76.35}
\\\hline\hline
  \end{tabular}
\end{table}

\begin{table}[t]
\caption{The impact of Angle-Axis map on accuracy (\%).}
\vspace{-6pt}
\centering
\begin{tabular}{@{\hspace{2pt}}c|c|c@{\hspace{2pt}}}\hline\hline
     Model & With Eq.~(\ref{angle-axis}) & Without Eq.~(\ref{angle-axis}) \\\hline
QMLP-LSTM   & \textbf{76.17} & 73.04 \\
QAGC-LSTM   & \textbf{79.13} & 75.65 \\
\hline\hline
\end{tabular}
\label{tab:2}
\end{table}

%% file: supp/supp.tex
\section*{Supplementary}

\begin{figure*}[!t]
    \centering
        \includegraphics[width=0.3\linewidth]{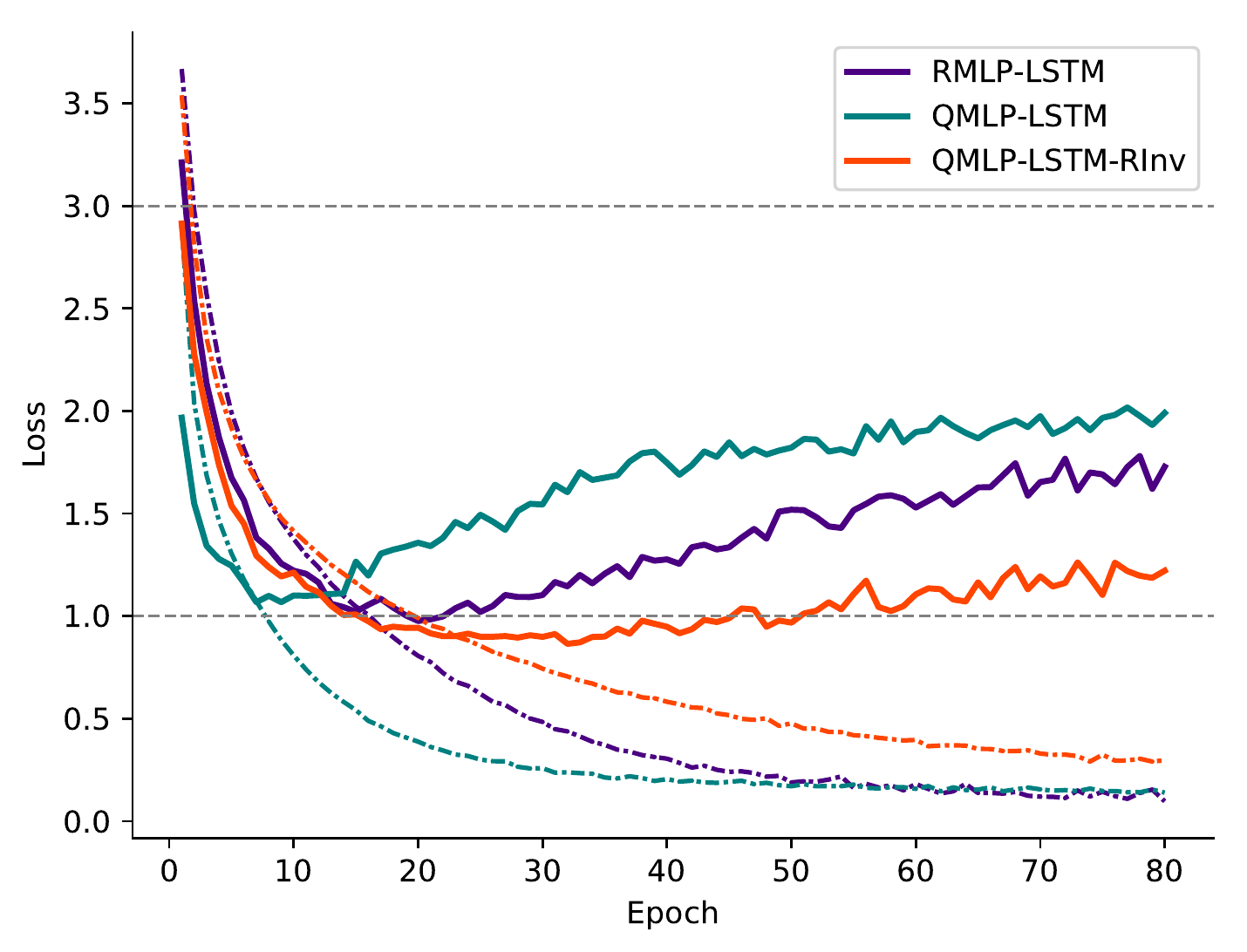}
        \includegraphics[width=0.3\linewidth]{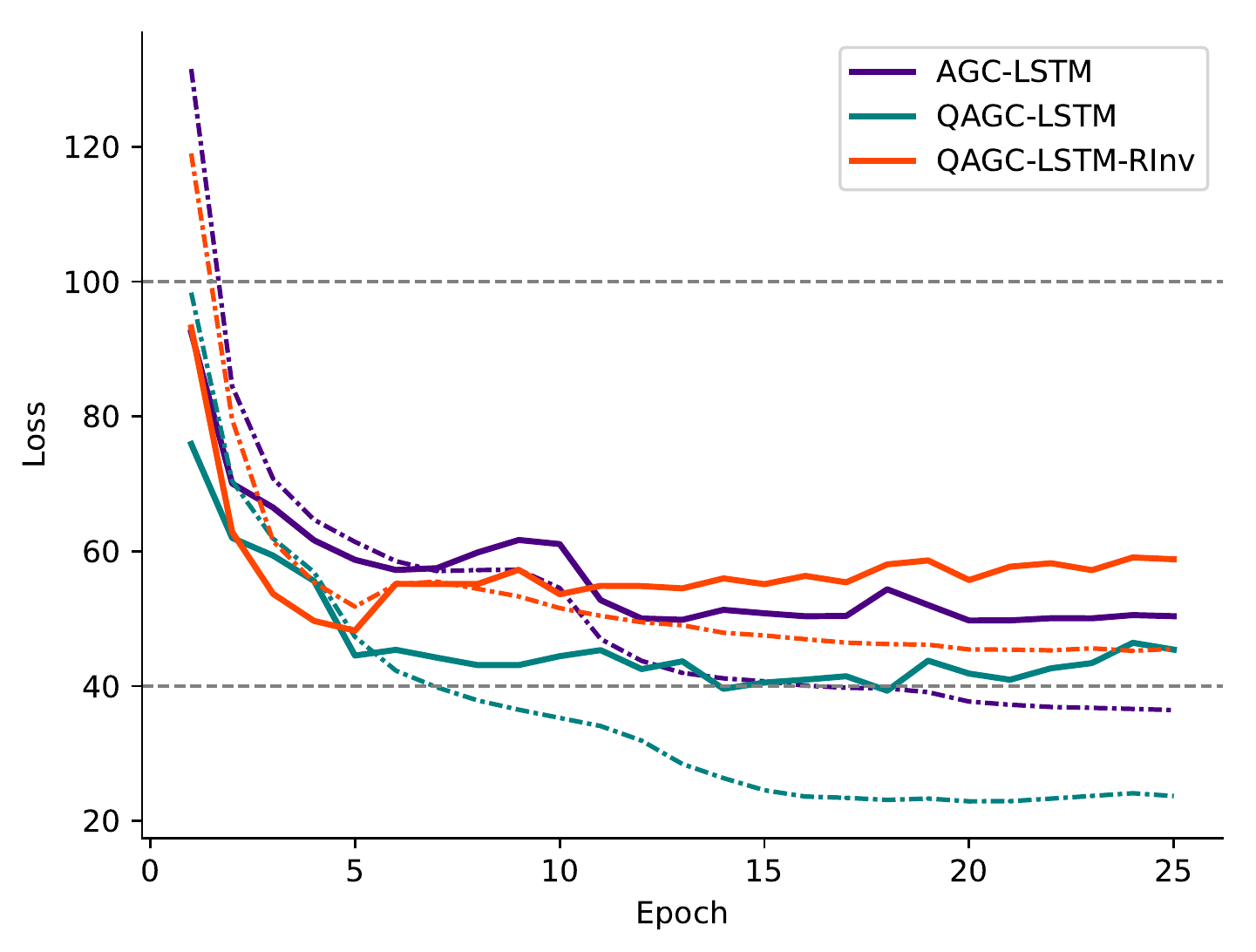}
        \includegraphics[width=0.3\linewidth]{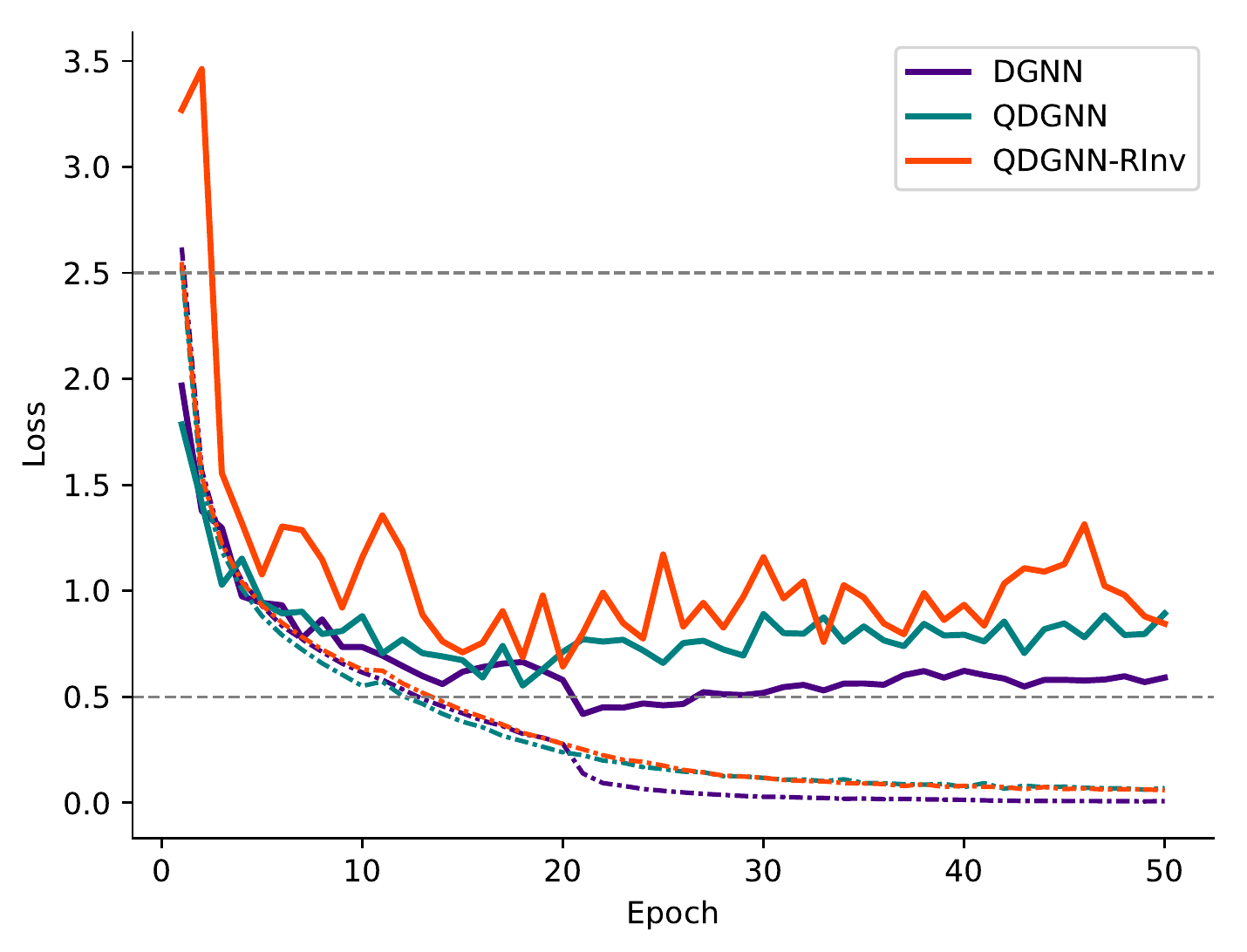}
    \vspace{-8pt}
    \caption{Comparisons on learning process.
    From left to right: MLP-LSTM, AGC-LSTM, DGNN.}
    \label{fig:loss}
\end{figure*}

\section{Effect of Angle-Axis Map on Accuracy}
Table~\ref{tab:angle-axis} shows the testing accuracy of all models using QPU-based layers with and without our angle-axis map. 
We perform the experiments on the FPHA dataset.
Using our angle-axis map is beneficial in most situations. 

\begin{table}[h]
  \caption{The impact of Angle-Axis map on accuracy (\%).}
  \vspace{-8pt}
  \centering
  \begin{tabular}{l | c  c}\hline\hline
    {Model}  & w/ angle-axis  & w/o angle-axis\\\hline
    QMLP-LSTM        & \textbf{76.17} & 73.04  \\
    QMLP-LSTM-RInv   & \textbf{68.00} & 64.17  \\\hline
    QAGC-LSTM        & \textbf{79.13} & 75.65  \\
    QAGC-LSTM-RInv   & 78.44          & \textbf{78.96} \\\hline
    QDGNN            & \textbf{82.26} & 80.87  \\
    QDGNN-RInv       & \textbf{76.35} & 74.26  \\\hline\hline
  \end{tabular}
  \label{tab:angle-axis}
\end{table}

\section{Effect of Rotation Data Augmentation}
As we mentioned in our main paper, data augmentation can also be used to enhance rotation robustness. 
To test the effect of rotation data augmentation, we train AGC-LSTM and QAGC-LSTM-RInv on an augmented NTU dataset, whose augmented samples are rotated randomly around $y$-axis\footnote{Here $y$-axis is the axis parallel with actor's spine.} and test them on the testing set with arbitrary rotations ($i.e.$, DA/AR). 
Table~\ref{tab:da} summarizes the results. 
We can find that applying QPU is more effective than applying data augmentation on enhancing rotation robustness --- the gains of testing accuracy caused by using QPU is much higher than those caused by training with augmented data.

\begin{table}[h]
  \caption{Comparisons on the gain of testing accuracy (\%).}
  \vspace{-8pt}
  \centering
  \begin{tabular}{l|cc|c}\hline\hline
    {Model}         
    &AR          
    &DA/AR   
    &Gain from DA
    \\\hline
    AGC-LSTM        
    &55.17 
    &66.64 
    &+11.47\\
    QAGC-LSTM-RInv  
    &\textbf{89.92} 
    &\textbf{90.07} 
    &+0.15
    \\\hline
    Gain from QPU
    &\textbf{+34.75}
    &\textbf{+23.43}
    \\\hline\hline
  \end{tabular}
  \label{tab:da}
\end{table}


\section{Training Procedures of Models}
We show in this section the training and evaluation loss of all tested models in Figure~\ref{fig:loss}, which verifies the convergence of our training method. 
The loss is derived from the experiments on the NTU dataset under the NR setting in the main paper. 
In all cases, evaluation losses of our QPU-based models, especially those with suffix ``RInv'', are comparable with those of real-valued models under the NR setting. 
Although their convergence rate is slightly slower than that of real-valued models, their robustness on rotations are much better, as we shown in our experiments.

\begin{figure}[t]
\centering
  \includegraphics[width=.7\linewidth]{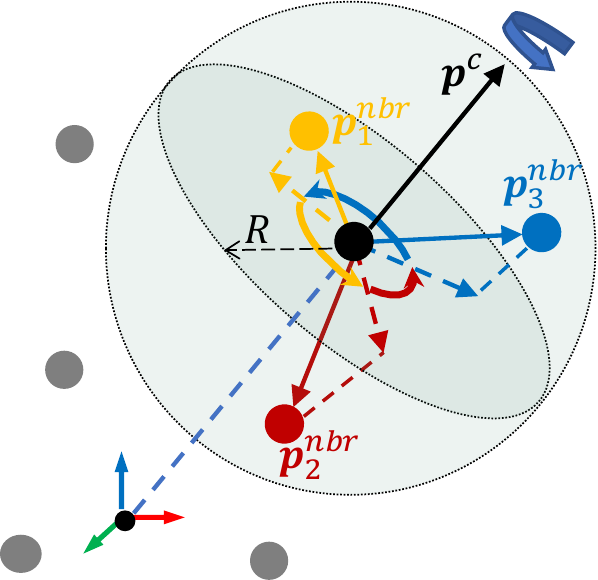}
 \caption{Illustration of neighbor points construction.
 For each centroid we group neighbor points within the radius $R$.
 $\pp^{c}$ is the vector from the origin to a centroid. 
 $\pp^{nbr}_{1,2,3}$ are vectors from the centroid to its neighbor points.
 We sort them cyclically by the clockwise rotation order of their projections on the plane orthogonal to $\pp^c$.
 In this case, $\pp^{nbr}_1$'s next point is $\pp^{nbr}_2$, $\pp^{nbr}_2$'s next point is $\pp^{nbr}_3$ and $\pp^{nbr}_3$'s next point is $\pp^{nbr}_1$.
 }
  \label{fig:pc}
\end{figure}

\section{Other Potential Applications}

We further test our QPU for rotation-invariant point cloud classification on the  ModelNet40 dataset~\citesupp{supp:wu20153d}, which contains normalized point clouds sampled from 40 types of CAD model. 
Each point cloud contains 1024 points. 
Following the Pointnet++ in~\citesupp{supp:qi2017pointnet++,supp:pytorchpointnet++}, we first sample 256 centroids for each point cloud. 
Then, for each centroid, we search 32 neighbor points in a fixed radius of 0.4. 
Accordingly, the relative 3D coordinates of each neighbor point with respect to its centroid, $i.e.$, $\pp^{nbr} =[x, y, z]_{\text{neighbor}} - [x,y,z]_{\text{centroid}}$, is represented as a quaternion-based 3D rotation, $i.e.$, $q = [\cos(\theta), \sin(\theta) \uu]$ by letting $\theta = \frac{\pi}{2}\cdot\frac{\norm{\pp^{nbr}}}{R}$ and $\uu = \frac{\pp^{nbr}}{\norm{\pp^{nbr}}}$, where $R$ is the searching radius.
Inspire by~\citesupp{supp:chen2019clusternet}, we cyclically sort each set of the neighbor points according to their rotation order around the vector from the origin to the centroid. 
Then for each point we concatenate its quaternion with the next 7 consecutive quaternions in the sorting so that each point contains information of its local structure.
Figure~\ref{fig:pc} illustrates the data preprocessing steps mentioned above.

For each point cloud, we take its quaternions as input and replace the first Set Abstration layer in the Pointnet++~\citesupp{supp:qi2017pointnet++, supp:pytorchpointnet++} with a QMLP module.
Specifically, the QMLP contains a QPU-based FC layer with 8$\times$4 input channels and 64 output channels and a real-valued FC layer with 64 input channels and 128 output channels. 
For each set of neighbor points, we pass its quaternion through this QMLP.
Two variants are tested, the first one is the model without rotation-invariance where we keep both the real parts and the imaginary parts of the outputs of QPU-based FC layer, the second one is the model with rotation-invariance where we multiply the output channels of QPU-based FC layer by 4 and only keep the real parts of the outputs of QPU-bsed FC layers.
We use a max-pooling layer to aggregate the real-valued outputs as the feature of the corresponding centroid. 
Finally, we obtain the representation of the point cloud by passing these centroids' features through two other Set Abstraction layers, each of which is composed of a MLP and a max-pooling layer. 

We train our models without rotation augmentation and test them under two settings: $i$) without arbitary rotation (NR), $ii$) with arbitary rotation (AR).
Using these features, we achieve 80.1\% test accuracy for the rotation-invariant model under both settings.  
For the model without rotation-invariance we achieve 90.0\% test accuracy under NR setting and 21.4\% under AR setting.
These results prove that our QPU together with our proposed point cloud representation can learn efficiently from point cloud data and that rotation-invariant classification can be achieved by keeping the real part of QPU's output.
We also tested our rotation-invariant model without the rotation ordering mentioned above and only achieve 75.9\% accuracy under both NR and AR settings.
This shows that the proposed rotation sorting is important for QPU to capture the structural information in neighbor point clouds.
We believe the accuracy can be further improved with more dedicated network architecture.
We also see QPU's potential in point cloud pose estimation.

{\small
\bibliographystylesupp{ieee_fullname}
\bibliographysupp{supp/ref_supp}
}
